\documentclass[12pt,a4paper,onecolumn]{IEEEtran} 

\usepackage{acronym}
\usepackage{caption} 

\acrodef{ADMM}[ADMM]{alternating direction method of multipliers}
\acrodef{AWGN}[AWGN]{additive white Gaussian noise}
\acrodef{PSNR}[PSNR]{peak signal-to-noise ratio}
\acrodef{GSP}[GSP]{graph signal processing}
\usepackage{etex}
\usepackage{amsmath,graphicx}
\usepackage{epstopdf}
\usepackage{endnotes}
\usepackage{hyperref}
\usepackage{setspace}

\usepackage{epsfig,psfrag}
\usepackage{pst-all}
\usepackage{amsmath,amsthm,amssymb,amsfonts,upref,cite,epsf,color,bm}
\usepackage{graphicx}
\usepackage{color}
\usepackage{amsmath}
\usepackage{graphicx}
\usepackage{calc}

\usepackage{tikz}
\usepackage{pgfplots}

\usepackage{algorithm}
\usepackage{algpseudocode}
\floatname{algorithm}{Algorithm}
\algnewcommand\algorithmicinput{\textbf{Input:}}
\algnewcommand\INPUT{\item[\algorithmicinput]}
\algnewcommand\algorithmicoutput{\textbf{Output:}}
\algnewcommand\OUTPUT{\item[\algorithmicoutput]}

\DeclareMathOperator*{\argmax}{arg\,max}

\DeclareMathOperator*{\argmin}{arg\;min}

\newcommand\vect[1]{\mathbf #1}
\newcommand\emperr[1]{{\rm Err} #1}
\newcommand{\va}{\vect{a}}  
\newcommand{\vb}{\vect{b}}

\newcommand{\ve}{\vect{e}}
  
\newcommand{\vg}{\vect{g}}

\newcommand{\vp}{\vect{p}}  
\newcommand{\vq}{\vect{q}}

\newcommand{\vu}{\vect{u}}  
\newcommand{\vv}{\hat{\vect{x}}}  

\newcommand{\vx}{\vect{x}}  
\newcommand{\vy}{\vect{y}}  
\newcommand{\vz}{\vect{z}}
\newcommand{\mP}{\mathbf{P}}

\newcommand{\mI}{\mathbf{I}}

\newcommand{\mA}{\mathbf{A}}
\newcommand{\mB}{\mathbf{B}}

\newcommand{\mW}{\mathbf{W}}

\newcommand{\signalsize}{N}
\newcommand{\nodespercluster}{100}
\newcommand{\samplespercluster}{100}
\newcommand{\nrcluster}{2}
\newcommand{\numiter}{2000}
\newcommand{\samplesize}{M}
\newcommand\defeq{:=}

\newtheorem{theorem}{Theorem}[section]
\newtheorem{lemma}[theorem]{Lemma}

\DeclareMathOperator{\diver}{div}

\title{Scalable Semi-Supervised Learning over Networks using Nonsmooth Convex Optimization\thanks{This
work was partially supported by the Vienna Science and Technology Fund (WWTF) 
under grant ICT15-119, Army Research Office grants W911NF-15-1-0479 
and W911NF-15-1-0241 and US Department of Energy grant DE-NA0002534.}\vspace*{2mm}}

\usepackage[T1]{fontenc}
\usepackage[utf8]{inputenc}
\usepackage{authblk}

\author[1]{Alexander Jung}
\author[2]{Alfred O.\ Hero III}
\author[1]{Alexandru Mara}
\author[3]{Sabeur Aridhi}
\affil[1]{\normalsize Dept.\ of CS, Aalto University, Finland; firstname.lastname(at)aalto.fi}
\affil[2]{\normalsize Dept.\ of EE and CS,The University of Michigan, MI; hero@eecs.umich.edu}

\affil[3]{\normalsize INRIA Nancy Grand Est, 54600 Villers-les-Nancy, France; sabeur.aridhi(at)loria.fr}
\begin{document}

\maketitle
\thispagestyle{plain}
\pagestyle{plain}

\begin{abstract}
We propose a scalable method for semi-supervised (transductive) learning 
from massive network-structured datasets. 
Our approach to semi-supervised learning is based on representing the underlying 
hypothesis as a graph signal with small total variation. 
Requiring a small total variation of the graph signal representing 
the underlying hypothesis corresponds to the central smoothness assumption 
that forms the basis for semi-supervised learning, i.e., input points forming clusters 
have similar output values or labels. 
We formulate the learning problem as a nonsmooth 
convex optimization problem which we solve by 
appealing to Nesterov's optimal first-order method for nonsmooth optimization. 
We also provide a message passing formulation of the learning method 
which allows for a highly scalable implementation in big data frameworks. 
\end{abstract}

	

	\section{Introduction}
	\label{sec_intro}

Modern technological systems generate (heterogeneous) data at unprecedented scale, i.e., 
``Big Data'' \cite{McKinseyBigdata,DonohoCursesBlessings,HadoopDefinitiveGuide,BigDataNetworksBook}. 
While lacking a precise formal definition, Big Data problems typically 
share four main characteristics: (i) large data volume, (ii) high speed of data generation, 
(iii) data is heterogeneous, i.e., partially labeled or unlabeled, 
mixture of audio, video and text data and (iv) data is noisy, i.e., there are statistical 
variations due to missing labels, labeling errors, or poor data curation \cite{HadoopDefinitiveGuide}. 
Moreover, in a wide range of big data applications, e.g., social networks, sensor networks, 
communication networks, and biological networks an intrinsic graph (or network) 
structure is present. This graph structure reflects either the physical properties of a system 
(e.g., public transportation networks) or statistical dependencies 
(e.g., probabilistic graphical models for bioinformatics). 
Quite often, these two notions of graph structure coincide: 
in a wireless sensor network, the graph modeling the communication links between nodes and 
the graph formed by statistical dependencies between sensor measurements 
resemble each other since both graphs are induced by the nodes mutual proximity 
\cite{WieselHero2012,Moldaschl2014,Quan20009}. 

On the algorithmic side, having a graph model for the observed datapoints 
faciliates scalable distributed data processing, in the form of message 
passing on the graph. On a higher-level, graph models are suitable to deal with data 
of diverse nature, since they only require a weak notion 
of similarity between datapoints. Moreover, graph models allow 
to capitalize on massive amounts of unlabeled data via semi-supervised learning. 
In particular, semi-supervised learning exploits the information contained in 
large amounts of unlabeled datapoints by considering 
their similarities to a small number of labeled datapoints. 

In this paper, we consider the problem of semi-supervised learning using  
a graph model for the raw data. The observed data consists of a small number 
of labeled datapoints and a huge amount of unlabeled datapoints. 
We tackle this learning problem by casting the dataset as a graph signal. In this 
graph signal model, the different dimensions of the data are identified as variables and the 
observed values of these variables are called signals. These signals are represented by 
nodes of a (empirical) graph whose edges represent pairwise dependencies between 
signals. 
Imposition of such graph signal structure on the data is analogous to making the 
{\emph smoothness assumption of semi-supervised learning} \cite{SemiSupervisedBook}: 
signals that are connected by an edge in the graph have similar labels. 
In other words, the graph signal is expected to reflect the underlying graph 
structure in the sense that the labels of signals on closely connected nodes 
have high mutual correlation and thus these signals form close-knit clusters or communities \cite{Fortunato2009}. 
In order to quantify the smoothness assumption underlying semi-supervised learning, one can 
use different measures to incorporate the topological dependency structure of graphs signals. 
For example, one can project the signals onto the column space of the graph Laplacian matrix, 
using the squared norm of the projected signals, i.e., the graph Laplacian form, as a measure 
of smoothness. This is the basis for many well-known label propagation methods \cite{SemiSupervisedBook}. 

In contrast, the approach proposed in this paper is 
based on using (graph) total variation \cite{shuman2013}, which 
provides a more natural match between smoothness and the community 
structure of the data, i.e., input or feature signal nodes forming a 
community or cluster should yield similar output values or labels.


\subsection{Contributions and Outline} 
In Section \ref{sec2_setup}, we formulate semi-supervised learning 
using a graph model for the observed data as a convex optimization 
problem. By adapting Nesterov's method for nonsmooth convex 
optimization, which is reviewed in Section \ref{sec3_smooth}, 
we propose an efficient learning algorithm in Section \ref{sec4_signal_recovery}. 
We then present a message passing formulation of our learning algorithm in 
Section \ref{sec6_message}, which only requires local information updating.
We also discuss how to implement the message passing formulation 
for graphs of massive size. 

\subsection{Notation}
Matrices are denoted by boldcase uppercase letters (e.g. $\mA$) and
column vectors are denoted by boldface lowercase letters (e.g. $\vx$). 
The $i$th entry of the vector $\vx$ is denoted by $x_i$. and the entry in the 
$i$th row and $j$th column of matrix $\mathbf{A}$ is $A_{i,}$.  
For vectors $\vx,\vy \!\in\! \mathbb{R}^{N}$ and matrices 
$\mathbf{X},\mathbf{Y}\!\in\! \mathbb{R}^{N \times N}$, we define 
the inner products $\langle \vx,\vy \rangle_2 \defeq \sum_i x_i y_i$ and
$\langle \mathbf{X}, \!\mathbf{Y} \rangle_{\rm F} \defeq \sum_{i,j} X_{i,j} Y_{i,j}$ 
with induced norms $\| \vx \|_2 \defeq  \sqrt{ \langle \vx, \vx \rangle_2 }$ and
$\| \mathbf{X}\|_{\rm F} \defeq  \sqrt{\langle \mathbf{X}, \mathbf{X} \rangle_{\rm F}}$.
For a generic Hilbert space $\mathcal{H}$, we denote its inner product by 
$\langle\cdot,\cdot \rangle_\mathcal{H}$. Given a linear operator $\mathbf{B}$ 
mapping the Hilbert space $\mathcal{H}_{1}$ into 
the Hilbert space $\mathcal{H}_{2}$, we denote its adjoint by 
$\mathbf{B}^{*}$ and by $\| \mathbf{B} \|_{\rm op} \defeq  \sup_{\| \vx \|_{\mathcal{H}_{1}} \leq 1} \| \mB \vx\|_{\mathcal{H}_{2}}$ 
its operator norm. The operator norm of a matrix $\mathbf{A} \in \mathbb{R}^{M \times N}$, interpreted as a mapping 
from Hilbert space $\mathbb{R}^{M}$ to $\mathbb{R}^{N}$, reduces to the spectral norm 
$\| \mathbf{A}\|_{2} \defeq  \sup_{ \vx \in {\mathbb{R}^{N}}\setminus \{\mathbf{0}\}} \| \mA \vx\|_{2} / \| \vx \|_{2}$. 
The $i$th column of the identity matrix $\mI$ is denoted by $\ve_i$. 
Given a closed convex subset $\mathcal{C}\subseteq \mathcal{H}$ of a Hilbert space, we denote
by $\pi_{\mathcal{C}}(\vx)\!=\!\underset{\vz \in \mathcal{C}}{\argmin} \|\vz - \vx\|_{\mathcal{H}}$ 
the orthogonal projection on $\mathcal{C}$. For a diagonal matrix $\mathbf{D}\in \mathbb{R}^{\signalsize\!\times\!\signalsize}$ 
with non-negative main diagonal entries $d_{i,i}$, we denote by $\mathbf{D}^{1/2}$ the diagonal matrix 
with main diagonal entries $\sqrt{d_{i,i}}$. 

\section{Problem Formulation}
\label{sec2_setup}

We consider a heterogeneous dataset $\mathcal{D}\!=\!\{ z_{i} \}_{i=1}^{\signalsize}\!\subseteq\!\mathcal{Z}$ 
consisting of $\signalsize$ datapoints $z_{i}\!\in\!\mathcal{Z}$, 
which might be of significantly different nature, e.g., $z_{1} \in \mathbb{R}^{d}$, $z_{2}$ is a continuous-time signal 
(i.e., $z_{2}: \mathbb{R}\!\rightarrow\!\mathbb{R}$) and $z_{3}$ might represent the bag-of-words histogram of 
a text document. Thus, we assume the input space $\mathcal{Z}$ rich enough to accomodate 
for strongly heterogeneous data. 
Associated with the dataset $\mathcal{D}$ is an undirected empirical graph 
$\mathcal{G} = (\mathcal{V}, \mathcal{E}, \mathbf{W})$ with node set $\mathcal{V}=\{1,\ldots,\signalsize\}$, 
edge set $\mathcal{E}\subseteq \mathcal{V} \times \mathcal{V}$ 
and symmetric weight matrix $\mathbf{W}\in \mathbb{R}^{\signalsize \times \signalsize}$. 
The nodes represent the datapoints, i.e., node $i$ corresponds to the datapoint $z_{i}$. 
An undirected edge $(i,j) \in \mathcal{E}$ encodes some notion of 
(physical or statistical) similarity from datapoint $z_{i}$ to datapoint $z_{j}$. 
Moreover, the presence of an edge $(i,j) \in \mathcal{E}$ between nodes $i,j \in \mathcal{V}$ 
is indicated by a nonzero entry $W_{i,j}=W_{j,i}$ of the weight matrix $\mW$. 
Given an edge $(i,j) \in \mathcal{E}$, the nonzero value $W_{i,j}\!>\!0$ 
represents the strength of the connection from node $i$ to node $j$. We assume 
the empirical graph to be simple, i.e., it contains no self-loops ($W_{i,i}\!=\!0$ for all $i\!\in\!\mathcal{V}$). 
The neighborhood $\mathcal{N}(i)$ and degree $d_{i}$ of node $i \in \mathcal{V}$ is defined, respectively, 
as 
\begin{equation} 
\label{equ_def_neighborhood}
\mathcal{N}(i) \defeq \{ j \in \mathcal{V} : (i,j)\!\in\!\mathcal{E} \}
\end{equation} 
 and   
\begin{equation} 
d_{i} \defeq \sum_{j \in \mathcal{N}(i)} W_{i,j}. \label{equ_def_node_degree}
\end{equation}
An key parameter for the characterization of a graph is the maximum node degree \cite{NewmannBook} 
\begin{equation}
\label{equ_def_max_node_degree}
 d_{\rm max} \defeq \max_{i \in \mathcal{V}} d_{i}. 
\end{equation} 

Within supervised machine learning, we assign to each 
datapoint $z_{i} \in \mathcal{D}$ an output value or label 
$x_{i} \in \mathbb{R}$.\footnote{We highlight that the term ``label'' is typically 
reserved for discrete-valued or categorial output variables $x_{i}$ \cite{BishopBook}. 
Since we can always represent the values of categorial output variables by real numbers, 
we will formulate our learning method for real-valued ouput variables $x_{i} \in \mathbb{R}$. 
Our learning method Alg.\ \ref{nestrov_alg}, which is based on using the 
squared error loss to quantify the empirical error, can also be used for 
classification by suitably quantizing the predicted ouput values. Extensions 
to other loss functions, more suitable to characterize the empirical 
error for discrete-valued or categorial labels, will be a focus of future work.} 
We emphasize that the label $x_{i}$ of node $i \in \mathcal{V}$ can take 
on binary values (i.e., $x_{i} \!\in\! \{0,1\}$), multi-level discrete values 
(i.e., $x_{i} \!\in\! \{1,\ldots,K\}$, with $K$ being the number of classes or clusters), 
or continuous values in $\mathbb{R}$.  
We can represent the entire labeling of the empirical graph conveniently 
by a vector $\vx \!\in\! \mathbb{R}^{\signalsize}$ whose $i$th entry is the 
label $x_{i}$ of node $i \in \mathcal{V}$. For a small subset $\mathcal{S}$ 
of datapoints $z_{i}$ we are provided with initial labels $y_{i}$. With slight 
abuse of notation, we refer by $\mathcal{S}$ also to the subset of nodes $i\in \mathcal{V}$ 
representing the datapoints $z_{i}$ for which initial labels $y_{i}$ are available.  
We refer to the set $\mathcal{S} \subseteq \mathcal{V}$ as the \emph{sampling set}, 
where typically $\samplesize \defeq |\mathcal{S}| \ll \signalsize$.

In order to learn the entire labeling $\vx$ from the initial labels $\{ y_{i} \}_{i \in \mathcal{S}}$, 
we invoke the basic smoothness assumption for semi-supervised learning \cite{SemiSupervisedBook}: 
If two points $z_1$, $z_2$ are close, with respect to a given topology on the input space $\mathcal{Z}$, 
then so should be the corresponding labels $x_1$, $x_2$, 
with respect to some distance measure on the label space $\mathbb{R}$.
For quantifying the smoothness of a labeling, 
we appeal to the discrete calculus for graph signals, which rests 
on the concept of a gradient for graph signals \cite[Sec.\ 13.2]{SemiSupervisedBook}. 
In order to draw on discrete calculus for quantifying smoothness of a 
labeling, we interpret the labels $x_{i}$, for $i \in \mathcal{V}$, as 
the values of a graph signal, i.e., a mapping $x[\cdot]: \mathcal{V} \rightarrow \mathbb{R}$ 
which maps node $i \!\in\! \mathcal{V}$ to graph signal value $x[i]\!=\!x_{i}$. 
Using this interpretation, we measure the smoothness 
of the labels via the (local) gradient $\nabla_{\!i} \mathbf{x}$ 
at node $i\!\in\!\mathcal{V}$, given as \cite{shuman2013} 
\begin{equation} 
\label{equ_def_local_gradient}
\big( \nabla_{\!i} \mathbf{x} \big)_{j} \!\defeq\! \sqrt{W_{i,j}}(x_{j}\!-\!x_{i}) .
\end{equation}
The norm $\|\nabla_{\!i}\vx\|_{2}\!=\!\sqrt{\sum_{j\in\mathcal{V}} W_{i,j} (x_{j}\!-\!x_{i})^2 }$ 
provides a measure for the local variation of the graph signal $\vx$ at node $i\!\in\!\mathcal{V}$. 
The (global) smoothness of the labels $x_{i}$ is then quantified by the total variation 
\cite{shuman2013}: 
\begin{equation} 
\label{equ_def_TV_norm}
\| \vx \|_{\rm TV} \defeq  \sum_{i \in \mathcal{V}} \| \nabla_{\!i} \mathbf{x} \|_{2}\!=\!\sum_{i \in \mathcal{V}} \sqrt{\sum_{j \in \mathcal{V}}W_{i,j}(x_{j}\!-\!x_{i})^{2} }. 
\end{equation} 
Note that the total variation \eqref{equ_def_TV_norm} is a seminorm, 
being equal to $0$ for labelings that are constant over connected graph components.

The basic idea of semi-supervised learning is to find a labeling $\vx$ of the 
datapoints $z_{i}$ by balancing the \emph{empirical error}
\begin{equation} 
\label{equ_def_emp_error}
\emperr[\vx] \defeq \sqrt{(1/2|\mathcal{S}|) \sum_{i \in \mathcal{S}} (x_{i}\!-\!y_{i})^{2}}, 
\end{equation}
which represents the deviation of the learned labels $x_{i}$ from the initial labels $y_{i}$, 
with the smoothness $\| \vx \|_{\rm TV}$. If we fix a maximum level $\varepsilon>0$ tolerated 
for the empirical error $\emperr[\vx]$, we can formulate semi-supervised learning 
as the optimization problem 
\begin{align}
\label{equ_min_constr}
\hat{\vx} & \!\in\! \argmin_{\vx \in \mathcal{Q}} \| \vx\|_{\rm TV} \nonumber \\ 
& \mbox{ with } \mathcal{Q}\!\defeq\!\{\vx \in \mathbb{R}^{\signalsize}\!: \emperr[\vx]\!\leq\!\varepsilon\}.
\end{align}
Since the objective function in \eqref{equ_min_constr} is the seminorm $\|\vz\|_{\rm TV}$, which is a convex function 
and also the constraint set $\mathcal{Q}$ is a convex set,\footnote{The seminorm $\|\vx\|_{\rm TV}$ 
is convex since it is homogeneous ($\| \alpha\vx\|_{\rm TV}\!=\!|\alpha|\|\vx\|_{\rm TV}$ for $\alpha \in \mathbb{R}$) 
and satisfies the triangle inequality ($\|\vx\!+\!\vy\|_{\rm TV} \!\leq\! \|\vx\|_{\rm TV}\!+\!\|\vy\|_{\rm TV}$).
These two properties imply convexity \cite[Section 3.1.5]{BoydConvexBook}.} 
problem \eqref{equ_min_constr} is a convex optimization problem.
As the notation in \eqref{equ_min_constr} suggests, and which can be verified by simple examples, 
there typically exist several solutions for this optimization problem. However, the methods we consider 
for solving \eqref{equ_min_constr} in the following do not require uniqueness of the solution, i.e., they 
work even if there are multiple optimal labelings $\hat{\vx}$. 

For completeness, we also mention an alternative convex formulation of the 
recovery problem \eqref{equ_min_constr}, based on using a penalty 
term for the total variation instead of constraining the empirical error: 
\begin{equation}
\label{equ_denoising_opt1}
\hat{\vx} \in  \argmin_{\vx \in \mathbb{R}^{\signalsize}} \emperr[\vx]  + \lambda \| \vx \|_{\rm TV}.  
\end{equation}
The regularization parameter $\lambda\!>\!0$ trades off small empirical risk $\emperr[\hat{\vx}]$ against 
small total variation $\| \hat{\vx} \|_{\rm TV}$ of the learned labeling $\hat{\vx}$. 

The convex optimization problems \eqref{equ_min_constr} and \eqref{equ_denoising_opt1} 
are related by convex duality \cite{BoydConvexBook, BertsekasNonLinProgr}: For each choice for $\varepsilon$
there is a choice for $\lambda$ (and vice-versa) such that the solutions of \eqref{equ_min_constr} 
and \eqref{equ_denoising_opt1} coincide. However, the relation between $\varepsilon$ and $\lambda$ for this 
equivalence to hold is non-trivial and determining the corresponding $\lambda$ for 
a given $\varepsilon$ is as challenging as solving the problem \eqref{equ_min_constr} itself \cite{becker2011nesta}.  
 
From a practical viewpoint, an advantage of the formulations \eqref{equ_min_constr} 
is that the parameter $\varepsilon$ may be interpreted as a noise level, which can be estimated or 
adjusted more easily than the parameter $\lambda$ 
of the learning problem \eqref{equ_denoising_opt1}. 
For the rest of the paper, we will focus on the learning problem
\eqref{equ_min_constr}.

Finally, for a dataset $\mathcal{D}$ whose empirical graph $\mathcal{G}$ is 
composed of several (weakly connected) components \cite{NewmannBook}, 
the learning problem \eqref{equ_min_constr} decompose into independent 
subproblems, i.e., one learning problem of the form \eqref{equ_min_constr} 
for each of the components. Therefore, we will henceforth, without loss 
of generality, consider datasets whose empirical graph $\mathcal{G}$ is (weakly) connected.

\vspace*{-4mm}
\section{Optimal Nonsmooth Convex Optimization} 
\label{sec3_smooth}
\vspace*{-2mm}

We will now briefly review a recently proposed method \cite{nestrov2005} for solving 
nonsmooth convex optimization problems, i.e., optimization problems 
with a non-differentiable objective function, such as \eqref{equ_min_constr}. 
This method exploits a particular structure, which is present in the problems 
\eqref{equ_min_constr}. In particular, this optimization method is based on (i) approximating 
a nonsmooth objective function by a smooth proxy and (ii) then 
applying an optimal first order (gradient based) method for minimizing this proxy. 

Consider a structured convex optimization problem of the generic form 
\begin{equation}\label{equ_structured_problem_Nesterov}
\hat{\vx}\!\in\!\underset{\vx\!\in\!\mathcal{Q}_1}{\argmin}~f(\vx) \defeq \hat{f}(\vx)\!+\!\underbrace{\underset{\vu \in \mathcal{Q}_2}{\max} \langle \vu,\mB\vx\rangle_{\mathcal{H}_2}\!-\!\hat{g}(\vu)}_{\defeq h_{0}(\vx)}.
\end{equation}
Here, $\mB:\mathcal{H}_1\rightarrow \mathcal{H}_2$ is a linear operator from a 
finite dimensional Hilbert space $\mathcal{H}_1$ to another 
finite dimensional Hilbert space $\mathcal{H}_2$, both defined over the real numbers. 
The set $\mathcal{Q}_1\subseteq \mathcal{H}_1$ is required to be a closed 
convex set and the set $\mathcal{Q}_2 \subset \mathcal{H}_2$ is a bounded, closed convex set. 
The functions $\hat{f}$ and $\hat{g}$ in \eqref{equ_structured_problem_Nesterov} are required 
to be continuous and convex on $\mathcal{Q}_1$ and $\mathcal{Q}_2$, respectively.
Moreover, the function $\hat{f}$ is assumed differentiable with gradient $\nabla \hat{f}$ being Lipschitz-continuous with constant $L\geq 0$, i.e., 
\begin{equation} \label{lip_cont}
\| \nabla \hat{f}(\vy)\!-\!\nabla \hat{f}(\vx) \|_{\mathcal{H}_1} \leq L \| \vy\!-\!\vx \|_{\mathcal{H}_1}. 
\end{equation}

\subsection{Smooth Approximation of Nonsmooth Objective}

In order to solve the nonsmooth problem \eqref{equ_structured_problem_Nesterov}, 
we approximate the non-differentiable component $h_{0}(\vx)$
by the smooth function
\begin{equation}\label{maxmu}
h_\mu(\vx)\!\defeq\!\underset{\vu\!\in\!\mathcal{Q}_2}{\max} \langle \vu,\mB\vx \rangle_{\mathcal{H}_2}\!-\!\hat{g}(\vu)\!-\!(\mu/2)\|\vu\|_{\mathcal{H}_2}^2
\end{equation}
with the smoothing parameter $\mu> 0$, yielding 
\begin{equation}\label{smoothed}
f_\mu(\vx)\!\defeq\!\hat{f}(\vx)\!+\!\underset{\vu \in \mathcal{Q}_2}{\max} \langle \vu,\mB\vx \rangle_{\mathcal{H}_2}\!-\!\hat{g}(\vu)\!-\!(\mu/2)\|\vu\|_{\mathcal{H}_2}^2. 
\end{equation} 
The objective function $f(\vx)$ of the original problem \eqref{equ_structured_problem_Nesterov} 
is obtained formally from \eqref{smoothed} for the choice $\mu=0$, i.e., $f(\vx) = f_{0}(\vx)$. 
Since the function $g(\vu)\!=\! \| \vu \|_{\mathcal{H}_2}^2$ is strongly 
convex, the optimization problem \eqref{maxmu} 
has a unique optimal point 
\begin{equation}\label{umu}
\mathbf{u}_\mu (\vx)\!=\!\underset{\mathbf{u} \in \mathcal{Q}_2}{\argmax}~ \langle \mathbf{u},\mB \vx\rangle_{\mathcal{H}_2}\!-\!\hat{g}(\vu)\!-\!(\mu/2)\|\mathbf{u}\|_{\mathcal{H}_2}^2.
\end{equation}
According to \cite[Theorem 1]{nestrov2005}, the function $h_{\mu}(\vx)$ (cf. \eqref{maxmu}) is differentiable with gradient  
\begin{equation}
\nabla h_\mu (\vx) = \mB^* \mathbf{u}_\mu(\vx), \nonumber
\end{equation}
which can be shown to be Lipschitz continuous with constant $(1/\mu)\|\mB\|_{\rm op}^2$.
Since the gradient $\nabla \hat{f}(\vx)$ of $\hat{f}(\vx)$ is 
assumed Lipschitz continuous with constant 
$L$ (cf. \eqref{lip_cont}), the function $f_\mu(\vx)$ (cf. \eqref{smoothed}) 
 has gradient
\begin{equation}\label{gradient_general}
\nabla f_\mu(\vx) = \nabla \hat{f}(\vx) + \mB^* \mathbf{u}_\mu(\vx)
\end{equation}
which is Lipschitz continuous with constant 
\begin{equation}\label{lmu}
L_\mu \defeq L + (1/\mu)\|\mB\|_{\rm op}^2.
\end{equation}

Furthermore, by evaluating \cite[Eq.\ (2.7)]{nestrov2005}, we have
\begin{equation}
\label{uniform}
f_{\mu}(\vx)\leq f_0(\vx) = f(\vx) \leq f_\mu(\vx) + (\mu/2)\underset{\vu\in\mathcal{Q}_2}{\max}~\|\vu\|_{\mathcal{H}_2}^2,
\end{equation}
which verifies that $f_\mu(\vx)$ is a uniform smooth approximation of 
the objective function $f(\vx)$ in \eqref{equ_structured_problem_Nesterov}.

By replacing the objective $f(\vx)$ in \eqref{equ_structured_problem_Nesterov} 
with its smooth approximation $f_{\mu}(\vx)$, 
we obtain the smooth optimization problem 
\begin{equation}\label{minfmu}
\hat{\vx}_{\mu} \in \underset{\vx\in \mathcal{Q}_{1} \subseteq \mathcal{H}_{1}}{\argmin}~f_\mu(\vx). 
\end{equation} 
The original nonsmooth problem \eqref{equ_structured_problem_Nesterov} 
is obtained formally from the smooth approximation \eqref{minfmu} for the 
particular choice $\mu=0$. For nonzero $\mu>0$, the solutions $\hat{\vx}$ 
of \eqref{equ_structured_problem_Nesterov} will be different from the 
solutions $\hat{\vx}_{\mu}$ of \eqref{minfmu} in general. However, 
for sufficiently small $\mu$ any solution $\hat{\vx}_{\mu}$ of \eqref{minfmu} 
will be also an approximate solution to \eqref{equ_structured_problem_Nesterov}.
We can relate the optimal values $f(\hat{\vx})$ and $f_{\mu}(\hat{\vx}_{\mu})$ 
of the original problem \eqref{equ_structured_problem_Nesterov} 
and its smooth approximation \eqref{minfmu}, respectively, with the help 
of \eqref{uniform}. Indeed, by inserting the optimal points $\hat{\vx}_{\mu}$ and 
$\hat{\vx}$ into the corresponding objective functions in \eqref{uniform}, we obtain 
\begin{equation}
\label{equ_uniform_optima}
f_{\mu}(\hat{\vx}_{\mu})\!\leq\!f(\hat{\vx}) \leq f_{\mu}(\hat{\vx}_{\mu})\!+\!(\mu/2) \underset{\vu\in\mathcal{Q}_2}{\max}~\|\vu\|_{\mathcal{H}_2}^2.
\end{equation} 
Thus, the optimal value $f_\mu(\hat{\vx}_\mu)$ of the smoothed 
problem \eqref{smoothed} provides an estimate for the optimal 
value $f(\hat{\vx})$ of the original problem \eqref{equ_structured_problem_Nesterov}.

\subsection{Optimal Gradient Method for Smooth Minimization}\label{nesteroviterations}
For solving the smooth optimization problem \eqref{minfmu}, being a proxy for the 
original nonsmooth problem \eqref{equ_structured_problem_Nesterov}, 
we apply an optimal first-order method \cite{nestrov2005,becker2011nesta}. 
This method achieves the optimal worst-case rate of convergence among all gradient 
based methods \cite{nestrov04,nestrov2005}. We summarize this 
method for solving \eqref{minfmu} in Alg.\ \ref{nestro}, 
which requires as input the smoothing parameter $\mu>0$, an initial guess $\vx_0$ 
and a valid Lipschitz constant $\hat{L}$ for the gradient \eqref{gradient_general}, i.e., 
satisfying $\hat{L}\geq L_\mu = L + (1/\mu) \|\mathbf{B}\|_{\rm op}^{2}$.
\begin{algorithm}[h]
\caption{Nesterov's algorithm for solving \eqref{minfmu}}{}
\begin{algorithmic}[1]
\renewcommand{\algorithmicrequire}{\textbf{Input:}}
\renewcommand{\algorithmicensure}{\textbf{Output:}}
\Require smoothing parameter $\mu$, initial guess $\vx_{0}$, 
Lipschitz constant $\hat{L}\geq L_{\mu}\!=\!L\!+\!(1/\mu) \|\mathbf{B}\|_{\rm op}^{2}$ (cf.\ \eqref{lmu}) 
\Statex\hspace{-6mm}{\bf Initialize:} iteration counter $k\!\defeq\!0$
\Repeat
\State $\vg_{k}\!\defeq\!\nabla f_{\mu} (\vx_{k})\!=\!\nabla \hat{f}(\vx)\!+\!\mB^* \mathbf{u}_\mu(\vx)$ with $\mathbf{u}_\mu(\vx)$ given by \eqref{umu} \label{algostep1} 
\vspace*{2mm}
\State $\vv_{k}\!\defeq\!\underset{\vx \in \mathcal{Q}_{1}}{\argmin} (\hat{L}/2) \| \vx\!-\!\vx_{k} \|_{\mathcal{H}_{1}}^{2} \!+\! \langle \vg_{k}, \vx\!-\!\vx_{k} \rangle_{\mathcal{H}_{1}}$
\vspace*{2mm}
\State $\vz_{k}\!\defeq\! \underset{\vx\in\mathcal{Q}_{1}}{\argmin} (\hat{L}/2) \| \vx\!-\!\vx_{0}\|_{\mathcal{H}_{1}}^{2}\!+\!\sum_{l=0}^k \frac{l\!+\!1}{2} \langle \vg_l, \vx\!-\!\vx_{l} \rangle_{\mathcal{H}_{1}}\label{un}$
\State $\vx_{k\!+\!1} \!\defeq\! \frac{2}{k\!+\!3} \vz_{k}\!+\!\Big(1\!-\!\frac{2}{k\!+\!3}\Big)\vv_k$ \label{algostep4}
\State $k\!\defeq\! k\!+\!1$
\Until{stopping criterion is satisfied}
\Ensure $\vv_k$
\end{algorithmic}
\label{nestro}
\end{algorithm}
For a particular stopping criterion of Alg.\ \ref{nestro}, one can monitor the relative decrease 
in the objective function $f_{\mu}(\vv_{k})$ \cite[Sec.\ 3.5.]{becker2011nesta}. 
Another option, which is used in our numerical experiments 
(cf.\ Section \ref{sec_num_exp}), is to run Alg.\ \ref{nestro} for 
a fixed number of iterations. 

The steps $2$ and $3$ of Alg.\ \ref{nestro} amount to computing the projected 
gradient descent step for the smooth optimization problem \eqref{minfmu}. However, what 
sets Alg.\ \ref{nestro} apart from standard gradient descent methods is step $4$. This 
step uses all previous gradient information in order to compute a projected minimizer $\vz_{k}$ 
of an increasingly more accurate approximation of the objective function $f_{\mu}(\vx)$. The new iterate 
$\vx_{k+1}$ is then obtained in step $5$ as a convex combination 
of the projected gradient descent step $\hat{\vx}_{k}$ and the minimizer $\vz_{k}$ 
of the approximation to the objective function. 

The output $\vv_{k}$ of Alg.\ \ref{nestro} satisfies \cite[Theorem 2]{nestrov2005} 
\begin{equation}
\label{equ_upper_bound_nesterov}
 f_{\mu} (\vv_k) - f_{\mu} (\hat{\vx}_{\mu})\leq \frac{2 \hat{L} \|\hat{\vx}_{\mu} - \vx_0\|_2^2}{(k+1)(k+2)}
\end{equation} 
for any optimal point $\hat{\vx}_{\mu}$ of \eqref{minfmu}. 
The convergence rate predicted by \eqref{equ_upper_bound_nesterov}, i.e., 
the error $f_\mu (\vv_k) - f_\mu (\hat{\vx}_{\mu})$ decaying proportional to 
$1/k^2$ with the iteration counter $k$,
is optimal among all gradient-based minimization methods 
for the class of continuously differentiable functions with 
Lipschitz continuous gradient \cite[Theorem 2.1.7]{nestrov04}.

The characterization \eqref{equ_upper_bound_nesterov} can be used to bound 
the number of iterations needed to run Alg.\ \ref{nestro} such that it delivers 
an approximate solution $\vv_{k} \in \mathcal{Q}_1$ for the nonsmooth problem 
\eqref{equ_structured_problem_Nesterov} with prescribed accuracy $\delta$, i.e., 
the output $\vv_k$ satisfies $f(\vv_{k}) - f(\hat{\vx}) \leq \delta$. 
\begin{lemma}\label{altern}
Let $\hat\vx \in \mathbb{R}^\signalsize$ and $\hat\vx_\mu\in \mathbb{R}^\signalsize$
be optimal points of the original problem \eqref{equ_structured_problem_Nesterov} and its 
smoothed proxy \eqref{minfmu}, respectively. Denote $D\defeq \underset{\vu\in\mathcal{Q}_2}{\max}~\|\vu\|_{\mathcal{H}_2}^2$ 
and assume Alg.\ \ref{nestro} is used with Lipschitz constant 
$\hat{L}\!=\!L\!+\!(1/\mu)\|\mB\|_{\rm op}^2$. 
Then, the output $\vv_{k}$ after $k$ iterations of Alg.\ \ref{nestro} satisfies
\begin{equation}\label{statement1}
f(\vv_{k}) - f(\hat{\vx})  \leq f_{\mu}(\vv_{k}) - f_\mu(\hat{\vx}_{\mu}) + (\mu/2) D.
\end{equation}
For the choice $\mu = \delta/D$, Alg.\ \ref{nestro} delivers 
a solution $\vv_{k}$ for the non-smooth problem \eqref{equ_structured_problem_Nesterov} 
with accuracy $\delta$, i.e., 
\begin{align} 
\label{equ_bound_acc}
 f(\vv_k)\!-\!f(\hat{\vx})  &\leq \delta \mbox{ for all } k\!\geq\!k_{\delta}  \\
k_{\delta}\!\defeq\!(2/\delta) &\|\hat{\vx}_{\mu}\!-\!\vx_{0}\|_{2} \sqrt{L\delta\!+\!D\|\mathbf{B}\|_{\rm op}^{2}}. \nonumber
\end{align} 
\end{lemma}
\begin{proof}
By combining \eqref{uniform} (for the choice $\vx = \vv_{k}$) with \eqref{equ_uniform_optima}, we have 
\begin{equation}
f(\vv_{k}) - f(\hat{\vx})  \leq f_{\mu}(\vv_{k}) - f_\mu(\hat{\vx}_{\mu}) + (\mu/2) D. 
\end{equation}
Choosing $\mu\!=\!\delta/D$ and using \eqref{equ_upper_bound_nesterov} for 
the particular choice $\hat{L}\!=\!L\!+\!(1/\mu)\|\mB\|_{\rm op}^2$ we obtain
\begin{align}
f(\vv_{k})\!-\!f(\hat{\vx})  & \leq \nonumber \\
& \hspace*{-25mm} (2/\delta)\|\hat\vx_{\mu}\!-\!\vx_0\|_2^2(L\delta\!+\!D \|\mB\|_{\rm op}^2)(1/k^2)\!+\!\delta/2,
\end{align}
which implies \eqref{equ_bound_acc}.
\end{proof}
According to Lemma \ref{altern} we need $k\propto 1/\delta$ iterations of 
Alg.\ \ref{nestro} for solving the nonsmooth optimization problem 
\eqref{equ_structured_problem_Nesterov} with accuracy $\delta$ (cf. \cite[Theorem 3]{nestrov2005}). 
This iteration complexity is essentially optimal for any first-order (sub-)gradient 
method solving problems of the form \eqref{equ_structured_problem_Nesterov} \cite{NemYudFOM}. 

The lower bound \eqref{equ_bound_acc}  on the iteration complexity of Alg.\ \ref{nestro} 
depends on both the desired accuracy $\delta$ (which is enforced by choosing the 
smoothing parameter as $\mu = \delta/D$) and the choice for the inital guess via 
$\|\hat{\vx}_{\mu}-\vx_{0}\|_{2}$. As discussed in \cite{becker2011nesta}, an effective approach 
to speed up the convergence of Alg.\ \ref{nestro} is to run it repeatedly with 
increasing accuracy (corresponding to decreasing values of the smoothing parameter $\mu$) 
and using the output of Alg.\ \ref{nestro} in a particular run as initial guess for the next run. Since 
the inital guesses used for Alg.\ \ref{nestro} in a new run becomes more accurate, it is possible 
to use a smaller value for the smooting parameter $\mu$, which effects an increased accuracy 
of the ouput of Alg.\ \ref{nestro} according to \eqref{equ_bound_acc}. However, a simpler option 
is to adapt the smoothing parameter directly ``on-the-fly'' within the iterations of Alg.\ \ref{nestro}. 
This results in Alg.\ \ref{acc_nestro} being an accelerated version of Alg.\ \ref{nestro}. 
\begin{algorithm}[h]
\caption{Accelerated Nesterov for solving \eqref{minfmu}}{}
\begin{algorithmic}[1]
\renewcommand{\algorithmicrequire}{\textbf{Input:}}
\renewcommand{\algorithmicensure}{\textbf{Output:}}
\Require initial smoothing parameter $\mu_{0}$, decreasing factor $\kappa$, initial guess $\vx_{0}$
\Statex\hspace{-6mm}{\bf Initialize:} iteration counter $k=0$
\Repeat
\State $\mu\!\defeq\!\mu_{0} \kappa^{k}$ 
\vspace*{2mm}
\State $\hat{L} \!\defeq\! L\!+\! (1/\mu) \|\mathbf{B}\|_{\rm op}^{2}$
\vspace*{2mm}
\State $\vg_{k} \!\defeq\! \nabla f_{\mu} (\vx_k) \!=\! \nabla \hat{f}(\vx)\!+\!\mB^* \mathbf{u}_\mu(\vx)$ with $\mathbf{u}_\mu(\vx)$ given by \eqref{umu} \label{algostep1} 
\vspace*{2mm}
\State $\vv_{k}\!\defeq\! \underset{\vx \in \mathcal{Q}_{1}}{\argmin} (\hat{L}/2) \| \vx\!-\!\vx_{k} \|_{\mathcal{H}_{1}}^{2} \!+\! \langle \vg_{k}, \vx\!-\!\vx_{k} \rangle_{\mathcal{H}_{1}}$
\vspace*{2mm}
\State $\vz_{k}\!\defeq\!\underset{\vx\in\mathcal{Q}_{1}}{\argmin} (\hat{L}/2) \| \vx\!-\!\vx_{0}\|_{\mathcal{H}_{1}}^{2}\!+\!\sum_{l=0}^k \frac{l\!+\!1}{2} \langle \vg_l, \vx\!-\!\vx_{l} \rangle_{\mathcal{H}_{1}}\label{un}$
\State $\vx_{k\!+\!1}\!\defeq\!\frac{2}{k\!+\!3} \vz_{k}\!+\!\Big(1\!-\!\frac{2}{k\!+\!3}\Big)\vv_k$ \label{algostep4}
\State $k \!\defeq\!k\!+\!1$
\Until{stopping criterion is satisfied}
\Ensure $\vv_k$
\end{algorithmic}
\label{acc_nestro}
\end{algorithm}

\section{Efficient Learning of Graph Signals}\label{sec4_signal_recovery}
We will now show that the semi-supervised learning problem \eqref{equ_min_constr} 
can be rephrased in the generic form \eqref{equ_structured_problem_Nesterov}. 
This will then allow us to apply Alg.\ \ref{nestro} for semi-supervised learning from big data, i.e., 
from high-dimensional heterogeneous data, over networks. 
To this end, we need to introduce the graph gradient 
operator $\nabla_{\!\mathcal{G}}$ as a mapping from the Hilbert space $\mathbb{R}^\signalsize$ 
endowed with inner product $\langle \va,\vb\rangle_2 \!=\! \va^T\vb$ into the Hilbert space 
$\mathbb{R}^{\signalsize\times \signalsize}$ endowed with inner product 
$\langle \mA,\mB\rangle_{\rm F}\!=\!\operatorname{Tr} \{\mA\mB^T\}$ \cite{JungSpawc2016,HannakAsilomar2016}. 
In particular, the gradient operator $\nabla_{\mathcal{G}}$ 
maps a graph signal $\vx\in\mathbb{R}^\signalsize$ to the matrix 
\begin{align} \label{equ_def_grad_matrix}
\nabla_{\!\mathcal{G}} \vx  \defeq
\begin{pmatrix} 
\nabla_{\!1} \mathbf{x},  \dots  ,\nabla_{\!\signalsize} \mathbf{x} 
\end{pmatrix}^T
\in \mathbb{R}^{\signalsize\times \signalsize}.
\end{align}
The $i$th row of the matrix $\nabla_{\!\mathcal{G}} \vx$ is given by the local gradient $\nabla_i\vx$ of 
the graph signal $\vx$ at node $i\in\mathcal{V}$ (cf. \eqref{equ_def_local_gradient}). 
Let us also highlight the close relation between the gradient operator $\nabla_{\mathcal{G}}:\mathbb{R}^{\signalsize} \rightarrow \mathbb{R}^{\signalsize \times \signalsize}$ and 
the normalized graph Laplacian matrix $\mathbf{L}$ \cite{ChungSpecGraphTheory}, defined element-wise as  
\begin{equation}
\big( \mathbf{L} \big)_{i,j} \defeq \begin{cases} 1 & \mbox{ if } i=j \mbox{ and } d_{i} \neq 0, \\ 
                                                         - 1/\sqrt{d_{i} d_{j}} & \mbox{ if } (i,j) \in \mathcal{E}, \\ 
                                                         0 & \mbox{ otherwise,}
                                                         \end{cases}
\end{equation} 
with $d_{i}$ being the degree of node $i \in \mathcal{V}$ (cf.\ \eqref{equ_def_node_degree}). 
If we define the diagonal matrix $\mathbf{D}$ with diagonal elements $d_{i,i} = d_{i}$, 
we have for any graph signal $\vx$ (cf.\ \cite[Eq.\ (6)]{shuman2013}) the identity 
\begin{equation}
\label{equ_quadratic_form_norm_L_nabla}
\| \nabla_{\!\mathcal{G}} \vx \|_{\rm F}^{2} = \mathbf{x}^{T} \mathbf{D}^{1/2} \mathbf{L} \mathbf{D}^{1/2} \mathbf{x}. 
\end{equation}  

We then define the divergence operator $\diver_{\mathcal{G}}: \mathbb{R}^{\signalsize \times \signalsize} \rightarrow \mathbb{R}^{\signalsize}$ 
as the negative adjoint of the gradient operator $\nabla_{\mathcal{G}}: \mathbb{R}^{\signalsize} \rightarrow \mathbb{R}^{\signalsize  \times \signalsize}$ (cf.\ \cite[Chapter 13]{SemiSupervisedBook}) , i.e., 
\begin{equation}
\label{equ_def_diver_neg_adj_nable}
\diver_{\mathcal{G}} \!\defeq\! - \nabla_{\mathcal{G}}^{*}.
\end{equation} 
A straightforward calculation (cf.\ \cite[Proposition 13.4]{SemiSupervisedBook}) reveals that the operator 
$\diver_\mathcal{G}$ maps a matrix $\mP\in \mathbb{R}^{\signalsize\times \signalsize}$ 
to the vector $\diver_\mathcal{G} \mP\in\mathbb{R}^\signalsize$ with entries 
\begin{align}
( \diver_\mathcal{G} \mathbf{P} )_{i} & =  \sum_{j \in \mathcal{V}} \sqrt{W_{i,j}} P_{i,j} - \sqrt{W_{j,i}} P_{j,i}  \nonumber \\ 
&  \stackrel{\eqref{equ_def_neighborhood}}{=} \sum_{j \in \mathcal{N}(i)} \sqrt{W_{i,j}} P_{i,j} - \sqrt{W_{j,i}} P_{j,i}  \label{equ_def_diver} . 
\end{align} 
We highlight the fact that both, the gradient $\nabla_{\mathcal{G}}: \mathbb{R}^{\signalsize} \rightarrow \mathbb{R}^{\signalsize \times \signalsize}$ as well 
as the divergence operator $\diver_\mathcal{G}$ 
depend on the graph structure due to the presence of the weights $W_{i,j}$ 
in \eqref{equ_def_local_gradient} and \eqref{equ_def_diver}. Moreover, the above 
definitions for the gradient and divergence operator over complex networks 
are straightforward generalizations of the well-known gradient and divergence 
operator for grid graphs representing 2D-images \cite{chambolle2004algorithm}. 

Using the identity $\| \nabla_{\!i}\vx \|_2 = \underset{\|\mathbf{p}_i\|_2 \leq 1}{\max} \langle \mathbf{p}_i,\!\nabla_{\!i}\vx \rangle_2$, 
we can represent the total variation \eqref{equ_def_TV_norm} as 
\begin{align}\label{tvconvex}
\| \vx \|_{\rm TV} 
 = \sum_{i \in \mathcal{V}}  \max_{\|\mathbf{p}_i\|_2 \leq 1} \langle \mathbf{p}_i,\nabla_{\!i}\vx \rangle_2
= \underset{\mathbf{P} \in \mathcal{P}}{\max}~ \langle \mathbf{P},\nabla_{\!\mathcal{G}}\vx\rangle_{\rm F} 
\end{align}
with the closed convex set 
\begin{align}\label{setP}
\mathcal{P} \defeq \{ \mathbf{P} = (\mathbf{p}_{1},\ldots,\mathbf{p}_{\signalsize})^{T} \in \mathbb{R}^{\signalsize \times \signalsize}: \\ 
 \|\mathbf{p}_{i} \|_{2} \leq 1 \mbox{ for every } i=1,\dots,\signalsize \}. \nonumber
\end{align}
Using \eqref{tvconvex}, the learning problem \eqref{equ_min_constr} can be written as
\begin{align}\label{minimize_t}
\underset{\vx\in \mathcal{Q} }{\min} ~f_0(\vx) \quad \text{with }
f_0(\vx) \defeq  \underset{\mathbf{P} \in \mathcal{P}}{\max}~ \langle \mathbf{P},\nabla_{\!\mathcal{G}} \vx\rangle_{\rm F}
\end{align}
with constraint set $\mathcal{Q} \!=\! \{\vx \in \mathbb{R}^{\signalsize}\!: 
\emperr[\vx] \leq \varepsilon\}$ (cf.\ \eqref{equ_def_emp_error} and \eqref{equ_min_constr}).
The optimization problem \eqref{minimize_t} is exactly of the form 
\eqref{equ_structured_problem_Nesterov} with the linear operator $\mB = \nabla_{\!\mathcal{G}}$, 
the functions $\hat{f}(\vx) \equiv 0$ and $\hat{g} (\vu)\equiv 0$, and the sets $\mathcal{Q}_1 = \mathcal{Q}$ 
and $\mathcal{Q}_2 = \mathcal{P}$. 
The smoothed version  (cf.\ \eqref{smoothed}) of the problem \eqref{equ_min_constr} 
is then obtained as
\begin{align}\label{minimize3_t}
&\underset{\vx\in \mathcal{Q} }{\min}~f_\mu(\vx) \quad\text{with }\\
& f_\mu(\vx) \defeq  \underset{\mathbf{P} \in \mathcal{P}}{\max} \left(\langle \mathbf{P},\!\nabla_{\!\mathcal{G}} \vx\rangle_{\rm F} \!-\!  (\mu/2)\|\mathbf{P}\|_{\rm F}^2\right). \nonumber
\end{align} 

In order to apply Alg.\ \ref{nestro} to the smoothed version \eqref{minimize3_t} 
of the learning problem \eqref{equ_min_constr}, 
we have to determine the gradient $\nabla f_\mu(\vx)$ and a corresponding 
valid Lipschitz constant $\hat{L}\geq L_\mu$ (cf.\ \eqref{lmu}). 
The gradient $\nabla f_{\mu}(\vx)$ is obtained by 
specializing \eqref{gradient_general} for the objective in \eqref{minimize3_t}, yielding 
\begin{equation}\label{gradfmu}
\nabla f_{\mu}(\vx)
 = -\diver_\mathcal{G} \mathbf{P}_{\mu} (\vx)
\end{equation}
with 
\begin{equation}\label{Pmu}
\mathbf{P}_\mu (\vx) =  \underset{\mathbf{P} \in \mathcal{P}}{\argmax}~ \left(\langle \mathbf{P},\nabla_{\!\mathcal{G}} \vx\rangle_{\rm F} -  (\mu/2)\|\mathbf{P}\|_{\rm F}^2\right).
\end{equation}
By the KKT conditions for constrained convex 
optimization problems \cite{JungSpawc2016,BoydConvexBook}, 
\begin{align}
\mathbf{P}_{\!\mu} (\vx) &= (\vq_{1},\ldots,\vq_{\signalsize})^{T} \\ 
\mbox{with } \vq_{i} & = \frac{1}{\max\{\mu,\|\nabla_{\!i}\vx\|_2\}} \nabla_{\!i}\vx  \nonumber
\end{align}
A particular Lipschitz constant for the gradient $\nabla f_{\mu}(\vx)$ 
is obtained, by specializing \eqref{lmu} to $\mathbf{B}=\nabla_{\!\mathcal{G}}$, as
\begin{equation}
L_\mu = (1/\mu)\|\nabla_{\!\mathcal{G}}\|_{\rm op}^2 \stackrel{\eqref{equ_def_diver_neg_adj_nable}}{=} (1/\mu) \|\diver_\mathcal{G}\|_{\rm op}^2.
\end{equation}
However, since evaluating the exact operator norm of the gradient (or divergence) 
operator is difficult for an arbitrary large-scale graph,\footnote{In many big data applications 
it is not possible to have a complete description of the graph, e.g. in form of an edge list, 
available. Instead, one typically has only knowledge about some basic parameters, e.g., 
the maximum node degree $d_{\rm max}$ (cf.\ \eqref{equ_def_max_node_degree}).} 
we will rely on a simple upper bound.
\begin{lemma}\label{lemmaopnorm}
Let $\mathcal{G}=(\mathcal{V},\mathcal{E},\mathbf{W})$ be a weighted undirected graph and let $\nabla_{\!\mathcal{G}}$
 denote the corresponding gradient operator \eqref{equ_def_grad_matrix}. 
The norm of the gradient operator satisfies 
\begin{equation}\label{opnorm}
\|\nabla_{\!\mathcal{G}}\|_{\rm op} \leq \sqrt{2 d_{\rm max}}
\end{equation}
with the maximum node degree $d_{\rm max}$ (cf.\ \eqref{equ_def_max_node_degree}). 
\end{lemma}
\begin{proof}
Due to \eqref{equ_quadratic_form_norm_L_nabla}, we have 
\begin{equation} 
\|\nabla_{\!\mathcal{G}}\|^{2}_{\rm op}  = \| \mathbf{D}^{1/2} \mathbf{L} \mathbf{D}^{1/2} \|_{2}, 
\end{equation}
which, since obviously $\| \mathbf{D}\|_{2} \leq d_{\rm max}$, implies 
\begin{equation} 
\|\nabla_{\!\mathcal{G}}\|^{2}_{\rm op} \leq d_{\rm max} \| \mathbf{L} \|_{2}. 
\end{equation} 
The bound \eqref{opnorm} follows from the well-known upper bound $\| \mathbf{L} \|_{2} \leq 2$ for 
the maximum eigenvalue (which is equal to the spectral norm) of the normalized 
Laplacian matrix $\mathbf{L}$ (cf.\ \cite[Lemma 1.7]{ChungSpecGraphTheory})
\end{proof}
According to Lemma \ref{lemmaopnorm}, the gradient $\nabla f_\mu(\vx)$ (cf.\ \eqref{gradfmu}) of $f_\mu(\vx)$ is Lipschitz 
with constant  
\begin{equation}\label{lipschitz1}
\hat{L}\!=\! 2 d_{\rm max}/\mu, 
\end{equation}
which we can use as input to Alg.\ \ref{nestro}.

In order to apply Alg.\ \ref{nestro} to the smoothed learning problem \eqref{minimize3_t}, 
we now present closed-form expressions for the updates in step $3$ and $4$ of Alg.\ 
\ref{nestro} for $\mathcal{Q}_{1}=\mathcal{Q}$ (cf.\ \eqref{equ_min_constr}).
\begin{lemma}
\label{lem_orthog_proj_closed_form}
Consider the convex set $\mathcal{Q}=\{\vx: \emperr[\vx] \leq \varepsilon\}$ and let
$\vy$ denote any labeling which is consistent with the initial labels $y_{i}$, i.e., $\big(\vy\big)_{i} = y_{i}$ for all $i \!\in\! \mathcal{S}$. 
Then, using the shorthand $\mathbf{D}(\mathcal{S}) \defeq \sum_{i \in \mathcal{S}} \mathbf{e}_{i} \mathbf{e}_{i}^{T}$, 
the solution $\vv_k$ of the optimization problem 
\begin{equation}
\label{equ_orthog_proj_closed_form_constr_opt1}
\vv_{k}\!=\!\underset{\vx \in \mathcal{Q}}{\argmin} (\hat{L}/2) \| \vx\!-\!\vx_{k} \|_{2}^{2} \!+\! \vg^{T}_{k} (\vx\!-\!\vx_{k})
\end{equation} 
is given by 
\begin{align}\label{equ_closed_form_step3}
\vv_{k}&\!=\!\big(\mathbf{I}\!+\!\lambda_{\varepsilon} \mathbf{D}(\mathcal{S}) \big)^{-1} (\vq\!+\!\lambda_{\varepsilon} \mathbf{D}(\mathcal{S}) \vy) \\ 
& \hspace*{-5mm}= (\mathbf{I}\!-\!\mathbf{D}(\mathcal{S}))\vq\!+\! 
\begin{cases}
\mathbf{D}(\mathcal{S})\big[\vy \!+\! (\varepsilon/ r)(\vq\!-\!\vy)\big]
&\hspace*{-3mm} \mbox{if }  r \!>\! \varepsilon \\
\mathbf{D}(\mathcal{S}) \vq & \hspace*{-10mm} \mbox{otherwise}
\end{cases} \nonumber
\end{align}
with
\begin{align}
\vq &\!\defeq\!\vx_{k}\!-\!(1/\hat{L}) \vg_{k} \mbox{, } r \!\defeq\!\emperr[\vq\!-\!\vy] \\ 
& \mbox{, and }  \lambda_{\varepsilon}\!\defeq\!\max\{0, (r/\varepsilon)\!-\!1\}. \nonumber
\end{align}
In a similar manner, the solution $\vz_{k}$ of the optimization problem 
\begin{equation}
 \vz_{k} \!=\! \underset{\vx\in\mathcal{Q}}{\argmin} (\hat{L}/2) \| \vx\!-\!\vx_{0}\|_{2}^{2}\!+\!(1/2)\sum_{l=0}^{k} (l\!+\!1) \vg^{T}_l (\vx\!-\!\vx_{l}) 
\end{equation} 
is given by 
\begin{align}\label{equ_closed_form_step4}
\vz_{k} &\!=\! \big(\mathbf{I} + \tilde{\lambda}_{\varepsilon} \mathbf{D}(\mathcal{S}) \big)^{-1} (\tilde{\vq}\!+\!\tilde{\lambda}_{\varepsilon} \mathbf{D}(\mathcal{S}) \vy) \\
&\hspace*{-4mm}\!=\! (\mathbf{I}\!-\!\mathbf{D}(\mathcal{S}))\tilde{\vq}\!+\! 
\begin{cases}
\mathbf{D}(\mathcal{S})\big[\vy \!+\! (\varepsilon/ r)(\tilde{\vq}\!-\!\vy)\big]
&\hspace*{-4mm}\mbox{if }  r \!>\! \varepsilon \\
\mathbf{D}(\mathcal{S}) \tilde{\vq} &\hspace*{-10mm}\mbox{otherwise}
\end{cases} \nonumber
\end{align}
with 
\begin{align}
\tilde{\vq}&\!\defeq\!\vx_{0}\!-\!(1/2\hat{L}) \sum_{l=0}^{k}(l\!+\!1) \vg_{l} \mbox{, } \tilde{r}\!\defeq\!\emperr[\tilde{\vq}\!-\!\vy] \\ 
& \mbox{, and }   \tilde{\lambda}_\varepsilon = \max\{0,(\tilde{r}/\varepsilon)\!-\!1\}.  \nonumber
\end{align}
\end{lemma}
\begin{proof}
see Appendix \ref{proof_lem_orthog_proj_closed_form}.  
\end{proof}
The closed-form expressions \eqref{equ_closed_form_step3} and \eqref{equ_closed_form_step4} are suitable modifications 
of those presented in \cite[Sec. 3]{becker2011nesta} 
to our setting of semi-supervised learning over complex networks. 
We are now in the position to specialize Alg.\ \ref{nestro} to the 
smoothed learning problem \eqref{minimize3_t} by 
using the closed-form expressions \eqref{equ_closed_form_step3} and \eqref{equ_closed_form_step4} 
for step $3$ and $4$ of Alg.\ \ref{nestro}. 
This results in Alg.\ \ref{nestrov_alg} for semi-supervised learning from 
big data over networks.

\begin{algorithm}[h!]
\caption{Semi-Supervised Learning via Nesterov's Method}
\begin{algorithmic}[1]
\renewcommand{\algorithmicrequire}{\textbf{Input:}}
\renewcommand{\algorithmicensure}{\textbf{Output:}}
\Require dataset $\mathcal{D}$ with empirical graph $\mathcal{G}$, 
subset $\mathcal{S} = \{i_1,\dots,i_\samplesize\}$ of datapoints with initial labels 
$\{y_{j}\}_{j\in\mathcal{S}}$, error level $\varepsilon$,
smoothing parameter $\mu$, initial guess for the labeling $\vx_0 \in \mathbb{R}^\signalsize$
\Statex\hspace{-6mm}{\bf Initialize:} $k\!\defeq\!0$, Lipschitz constant $\hat{L}\!\defeq\!(2/\mu)d_{\rm max}$, $\tilde{\vq}_{0}\!\defeq\!\vx_{0}$, $\alpha\!\defeq\!1/2$
 \Repeat
\State $\forall i\in\mathcal{V}:$ $\vp_{i} \defeq  \frac{1}{\max\{\mu,\|\nabla_{\!i}\vx_k\|_2\}} \nabla_{\!i}\vx_k$ 
\State $\vg_{k} \!\defeq\!  -\diver_\mathcal{G} (\mathbf{P})$ with $\mP\!=\!(\vp_1,\dots,\vp_\signalsize)^T$ 
\State $\vq_{k} \!\defeq\!  \vx_{k}\!-\!(1/\hat{L})\vg_{k}$ 
\State $r \!\defeq\! \emperr[\vq_k]$ (cf.\ \eqref{equ_def_emp_error})
\State $\forall i \in \mathcal{V}:$  
$
\hat{x}_{k,i} \defeq
\begin{cases}
 y_i + (\varepsilon/ r)(q_{k,i} - y_i)
&\mbox{if } i \in \mathcal{S} \mbox{ and } r> \varepsilon\\
q_{k,i} &\mbox{otherwise}
\end{cases}
$

\State $\tilde{\vq}_{k} \defeq \tilde{\vq}_{k}\!-\!(\alpha/\hat{L}) \vg_{k} $ 
\State $\tilde{r} \defeq \emperr[\tilde{\vq}_{k}]$
\State $\forall i \in \mathcal{V}:$ 
$
z_{k,i} \defeq
\begin{cases}
 y_{i}\!+\!(\varepsilon/ \tilde{r})(\tilde{q}_{k,i}\!-\!y_i) 
&\mbox{if } i \in \mathcal{S} \mbox{ and } \tilde{r}> \varepsilon\\
\tilde{q}_{k,i} &\mbox{otherwise}
\end{cases}
$		
\State $\vx_{k\!+\!1} \defeq  \frac{2}{k\!+\!3} \vz_{k}\!+\!(1\!-\!\frac{2}{k\!+\!3})\vv_{k}$
\State $k \defeq k\!+\!1$
\State $\alpha \defeq \alpha +1/2$ 
\Until{stopping criterion is satisfied}
\Ensure learned labeling $\vv_k$ for all datapoints
\end{algorithmic}\label{nestrov_alg}
\end{algorithm}
The steps $2$ and $3$ of Alg.\ \ref{nestrov_alg} amount to computing the gradient $\vg_{k} = \nabla f_{\mu} (\vx_k)$ 
of the objective function $f_{\mu}(\vx)$ in the learning problem \eqref{minimize3_t}. The steps $4$-$6$ of Alg.\ \ref{nestrov_alg}
implement a projected gradient descent step, while steps $7$-$9$ amount to computing the minimizer 
$\vz_{k}$ of the approximation in step $4$ of Alg.\ \ref{nestro}. 

Combining \eqref{equ_upper_bound_nesterov} with \eqref{lipschitz1}, 
yields the following characterization of the 
convergence rate of Alg.\ \ref{nestrov_alg}: 
\begin{equation}
\label{equ_bound_conve_graph_topology}
 f_{\mu} (\vv_k) - f_{\mu} (\hat{\vx}_{\mu})\leq \frac{4 d_{\rm max} \|\hat{\vx}_{\mu} - \vx_0\|_2^2}{\mu(k+1)(k+2)}
\end{equation} 
for any solution $\hat{\vx}_{\mu}$ of \eqref{minimize3_t}. 
The bound \eqref{equ_bound_conve_graph_topology} suggests that the convergence is 
faster for graphs which are more sparse, i.e., 
have a smaller maximum node degree $d_{\rm max}$. 
As for Alg.\ \ref{nestro}, the convergence speed of Alg.\ \ref{nestrov_alg} depends on the accuracy 
of the inital guess as well as on the smoothing parameter $\mu$. The accelerated 
version of Alg.\ \ref{nestrov_alg} is then obtained from Alg.\ \ref{acc_nestro}, yielding 
Alg.\ \ref{nestrov_alg_acc}.
\begin{algorithm}[h!]
\caption{Semi-Supervised Learning via Accelerated Nesterov}
\begin{algorithmic}[1]
\renewcommand{\algorithmicrequire}{\textbf{Input:}}
\renewcommand{\algorithmicensure}{\textbf{Output:}}
\Require dataset $\mathcal{D}$ with empirical graph $\mathcal{G}$, 
subset $\mathcal{S} = \{i_1,\dots,i_\samplesize\}$ of datapoints with initial labels 
$\{y_{j}\}_{j\in\mathcal{S}}$, error level $\varepsilon$, 
initial smoothing parameter $\mu_{0}$, decreasing factor $\kappa$, 
initial guess for the labeling $\vx_0 \in \mathbb{R}^\signalsize$
\Statex\hspace{-6mm}{\bf Initialize:} $k\!\defeq\!0$, $\tilde{\vq}_{0}\!\defeq\!\vx_{0}$, $\alpha\!\defeq\!1/2$
 \Repeat
\State $\mu \!\defeq\! \mu_{0} \kappa^{k}$ 
\State $\hat{L}\!\defeq\!(2/\mu)d_{\rm max}$
\State $\forall i\in\mathcal{V}:$ $\vp_{i} \defeq  \frac{1}{\max\{\mu,\|\nabla_{\!i}\vx_k\|_2\}} \nabla_{\!i}\vx_k$ 
\State $\vg_{k} \!\defeq\!  -\diver_\mathcal{G} (\mathbf{P})$ with $\mP\!=\!(\vp_1,\dots,\vp_\signalsize)^T$ 
\State $\vq_{k} \!\defeq\!  \vx_{k}\!-\!(1/\hat{L})\vg_{k}$ 
\State $r \!\defeq\! \emperr[\vq_k]$ (cf.\ \eqref{equ_def_emp_error})
\State $\forall i \in \mathcal{V}:$  
$
\hat{x}_{k,i} \defeq
\begin{cases}
 y_i + (\varepsilon/ r)(q_{k,i} - y_i)
&\mbox{if } i \in \mathcal{S} \mbox{ and } r> \varepsilon\\
q_{k,i} &\mbox{otherwise}
\end{cases}
$

\State $\tilde{\vq}_{k} \defeq \tilde{\vq}_{k}\!-\!(\alpha/\hat{L}) \vg_{k} $ 
\State $\tilde{r} \defeq \emperr[\tilde{\vq}_{k}]$
\State $\forall i \in \mathcal{V}:$ 
$
z_{k,i} \defeq
\begin{cases}
 y_{i}\!+\!(\varepsilon/ \tilde{r})(\tilde{q}_{k,i}\!-\!y_i) 
&\mbox{if } i \in \mathcal{S} \mbox{ and } \tilde{r}> \varepsilon\\
\tilde{q}_{k,i} &\mbox{otherwise}
\end{cases}
$		
\State $\vx_{k\!+\!1} \defeq  \frac{2}{k\!+\!3} \vz_{k}\!+\!(1\!-\!\frac{2}{k\!+\!3})\vv_{k}$
\State $k \defeq k\!+\!1$
\State $\alpha \defeq \alpha +1/2$ 
\Until{stopping criterion is satisfied}
\Ensure learned labeling $\vv_{k}$ for all datapoints
\end{algorithmic}\label{nestrov_alg_acc}
\end{algorithm}

\section{Message passing formulations}\label{sec6_message}
In order to make semi-supervised learning via the optimization problem  
\eqref{minimize3_t} feasible for massive (internet-scale) datasets, 
we will now discuss a message passing formulation of Alg.\ \ref{nestrov_alg}, 
which is summarized as Alg.\ \ref{MP_nestrov_alg} (being an adaption of \cite[Alg.\ 2]{HannakAsilomar2016} 
to undirected graphs): 
The steps 2-6 of Alg.\ \ref{MP_nestrov_alg} amount to computing the (scaled) gradient $\nabla f_{\mu}(\vx)$ 
(cf. \eqref{gradfmu}) of the objective function $f_{\mu}(\vx)$ for the problem  
\eqref{minimize3_t} in a distributed manner. 
The quantities $P_{i,j}$ are the entries of the matrix $\mP_\mu(\vx)$ (cf.\ \eqref{Pmu}). 
The steps 11-15 and 17-21 of Alg.\ \ref{MP_nestrov_alg} implement a finite number $K$ 
of iterations of the average consensus algorithm, 
using Metropolis-Hastings weights \cite{Xiao07} , for (approximately) computing the 
sums $(1/\signalsize)\sum_{j \in\mathcal{V}}b_j = (1/\signalsize)\sum_{j \in\mathcal{S}}(y_j - q_j)^2$ 
and $(1/\signalsize)\sum_{j \in\mathcal{V}}\tilde{b}_j = (1/\signalsize)\sum_{j \in\mathcal{S}}(y_j - \tilde{q}_j)^2$, 
respectively. In particular, for sufficiently large $K$, the results $r_i$ and $\tilde{r}_i$ 
in step 16 and 22 of Alg.\ \ref{MP_nestrov_alg} satisfy $r_i\approx r$ and 
$\tilde{r}_i \approx \tilde{r}$ for every node $i\in\mathcal{V}$ (cf.\ step 6,7 of Alg.\ \ref{nestrov_alg}). 
The choice for the number $K$ of avarage consensus iterations can be guided by a wide range of 
results characterzing the convergence rate of average conensus \cite{BoydFastestMixing2004,Xiao07,Diaconis1998}. 
As a rule of thumb, $K$ should be significantly larger than the diameter $d(\mathcal{G})$ of 
the underlying graph $\mathcal{G}$ \cite[Thm. 4.3.]{Diaconis1998}. 
We highlight that Alg.\ \ref{MP_nestrov_alg} requires each node $i\in\mathcal{V}$ 
to have access to local information only. In particular, to implement Alg.\ \ref{MP_nestrov_alg} 
on a given node $i\in\mathcal{V}$, the measurement $y_j$, the value $x_j$, 
the matrix entries $P_{i,j}$ and the edge weights $W_{i,j}$ are required 
only for node $i$ itself and its neighborhood $\mathcal{N}(i)$.

\begin{algorithm}[h]
	\caption{Distributed Semi-Supervised Learning via Nesterov's Method}{}
	\begin{algorithmic}[1]
		\vspace*{1mm}
\renewcommand{\algorithmicrequire}{\textbf{Input:}}
\renewcommand{\algorithmicensure}{\textbf{Output:}}
		\Require{ sampling set $\mathcal{S} = \{i_1,\dots,i_\samplesize\}$, samples
$\{y_{j}\}_{j\in\mathcal{S}}$, error level $\varepsilon$, edge weights $\{W_{i,j}\}_{i,j\in\mathcal{V}}$,
smoothing parameter $\mu$,
initial guess $\{x_{0,i}\}_{i\in\mathcal{V}}$, number $K$ of average consensus iterations}
	\Statex\hspace{-6mm}{\bf Initialize:} $\hat{L} \!\defeq\! (2/\mu)d_{\rm max}$, 
	$\forall i \in \mathcal{V}:x_{i}\!=\!x_{0,i}$, $\alpha\!\defeq\!1/2$, $\tau\!\defeq\!2/3$,
	$\{u_{i,j} \!\defeq\!1/(\max\{d_i,d_j\}\!+\!1)\}_{j\in\mathcal{N}(i)}$
\Repeat
\State $\forall i\in\mathcal{V}:$  broadcast $x_i$ to neighbors 
$\mathcal{N}(i)$/ collect $\{x_j\}_{j\in \mathcal{N}(i)}$ from neighbors $\mathcal{N}(i)$
 \State 
$\forall i\in\mathcal{V}:$ update 
$
\gamma_i = \sqrt{\sum_{j\in\mathcal{N}(i)} (x_j - x_i)^2 W_{i,j}}$
\State $\forall i \in \mathcal{V}:$ for neighbor $j \in \mathcal{N}(i)$ update 
$ 
		P_{i,j} = \frac{1}{\max\{\mu,\gamma_i\}} (x_j - x_i) \sqrt{W_{i,j}}
$
\State $\forall i\in\mathcal{V}:$ broadcast $P_{i,j}$  to neighbors 
$\mathcal{N}(i)$/ collect $\{P_{j,i}\}_{j\in \mathcal{N}(i)}$ from neighbors $\mathcal{N}(i)$
\State $\forall i\in\mathcal{V}:$ update 
$
g_i = (1/\hat{L}) \big( \sum_{j\in \mathcal{N}(i)} \sqrt{W_{j,i}}P_{j,i} -\sum_{j\in \mathcal{N}(i)}\sqrt{W_{i,j}}P_{i,j} \big)
$
 \State $\forall i\in\mathcal{V}:$  update $q_i = x_i - g_i$
\State $\forall i\in\mathcal{V}:$ update $\overline{g}_i = \overline{g}_i - \alpha g_i$
\State $\forall i\in\mathcal{V}:$ update $\alpha = \alpha+1/2$
\State $\forall i\in\mathcal{V}:$ update 
    $\tilde{q}_i = x_{0,i} \!-\! \overline{g}_i$ 
\State $\forall i\in\mathcal{V}:$ update 
$	b_i =
		\begin{cases}
		(y_i - q_i)^2 &\mbox{ for } i\in \mathcal{S}\\
		0  &\mbox{ else}
		\end{cases}$
\For{\texttt{l = 1:K}}
        \State $\forall i\in\mathcal{V}:$  broadcast $b_i$ to neighbors  $\mathcal{N}(i)$/ collect $\{b_j\}_{j\in \mathcal{N}(i)}$ from $\mathcal{N}(i)$
 \State $\forall i\in\mathcal{V}:$  update 
$b_i  = \big(1 - \sum_{j \in \mathcal{N}(i)} u_{i,j} \big) b_i +
\sum_{j\in \mathcal{N}(i)}u_{i,j} b_j$
      \EndFor
\State $\forall i\in\mathcal{V}:$ update $r_i = \sqrt{\signalsize b_i}$
\State $\forall i\in\mathcal{V}:$ update 
$	\tilde{b}_i =
		\begin{cases}
		(y_i - \tilde{q}_i)^2 &\mbox{ for } i\in \mathcal{S}\\
		0  &\mbox{ else}
		\end{cases}$
\For{\texttt{l = 1:K}}
\State $\forall i\in\mathcal{V}:$  broadcast $\tilde{b}_i$ to neighbors 
$\mathcal{N}(i)$/collect $\{\tilde{b}_j\}_{j\in \mathcal{N}(i)}$ from $\mathcal{N}(i)$
 \State $\forall i\in\mathcal{V}:$  update 
$\tilde{b}_i  = \big(1 - \sum_{j \in \mathcal{N}(i)} u_{i,j} \big) \tilde{b}_i +
\sum_{j\in \mathcal{N}(i)}u_{i,j} \tilde{b}_j$
      \EndFor
\State $\forall i\in\mathcal{V}:$ update $\tilde{r}_i= \sqrt{\signalsize\tilde{b}_i}$	
\State $\forall i \in \mathcal{V}:$ update 
$
\hat{x}_i =
\begin{cases}
 y_i + (\varepsilon/ r)(q_i -y_i)
&\mbox{if } i \in \mathcal{S} \mbox{ and } r> \varepsilon\\
q_i &\mbox{otherwise}
\end{cases}
$
\State $\forall i \in \mathcal{V}:$ update 
$
z_i =
\begin{cases}
 y_i + (\varepsilon/ \tilde{r})(\tilde{q}_i -y_i) 
&\mbox{if } i \in \mathcal{S} \mbox{ and } \tilde{r}> \varepsilon\\
\tilde{q}_i &\mbox{otherwise}
\end{cases}
$				
\State $\forall i\in\mathcal{V}:$ update $x_i = \tau z_i + (1-\tau)\hat{x}_i$
\State $\forall i\in\mathcal{V}:$ update $\tau = \Big( (1/\tau)+ (1/2)\Big)^{-1}$
\Until{stopping criterion is satisfied}
	\Ensure $\hat{x}_i$
	\end{algorithmic}\label{MP_nestrov_alg}
	\end{algorithm}


We implemented Alg.\ \ref{MP_nestrov_alg} using the big data framework \textsc{akka} \cite{accabook}, 
which is a toolkit for building distributed and resilient message-driven applications.
The \textsc{akka} implementation was run on a computing cluster 
composed of nine virtual machines (one master and eight slave workers) 
obtained via the cloud computing service Amazon EC2. 
Each virtual machine has been configured with a 64-bit CPU, $3.75$ 
GB of main memory, and $8$ GB of local disk storage.
In Figure \ref{fig:acca1} we sketch the basic architecture 
of the \textsc{akka} implementation of our graph learning methods. 
First, we partition the graph in a simple uniform manner, i.e., 
node $i\in \{1,\ldots,N\}$ is assigend to partition $(i \bmod 8)+1$. 
After partitioning $\mathcal{G}$, the master machine 
assigns the obtained partitions to the eight workers and 
manages the execution of the message passing algorithm 
between the workers. There are two alternating phases 
in the execution of the \textsc{akka} implementation: 
the master phase, where the states of the 
worker machines are synchronized and the worker phase. 
Two types of operations are executed in the worker phase:
\begin{itemize}
	\item intra-block operations: each worker performs local computations within its associated partition, and
	\item inter-block operations: workers exchange messages across their partitions.
\end{itemize}
We then compared the runtime of the AKKA implementation to the 
centralized implementation in MATLAB used in \cite{HannakAsilomar2016}. 
The results indicate a runtime reduction by almost a factor $10$ 
which is reasonable since we are using a cluster of nine machines. 

\begin{figure}[h]
	\centering
	\hspace*{0em}\includegraphics[width=.5\linewidth]{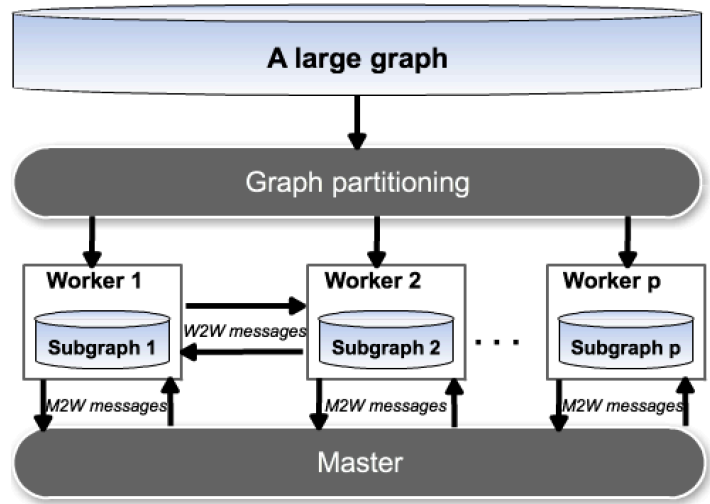}
	\captionsetup{justification=centering}
	\caption{\textsc{akka} big data framework overview.} \label{fig:acca1}
\end{figure}

\section{Numerical Experiments} 
\label{sec_num_exp}

We assess the performance of (the accelerated version of) the proposed learning algorithm 
Alg.\ \ref{nestrov_alg_acc} empirically by applying it 
to an synthetic dataset with empirical graph 
$\mathcal{G}=(\mathcal{V},\mathcal{E},\mathbf{W})$, depicted in Fig.\ \ref{fig_emp_grpah}, 
whose nodes 
are made up of $\nrcluster$ disjoint clusters $\mathcal{C}_{c}$ of same size 
$|\mathcal{C}_{c}| \!=\! \nodespercluster$ giving a total graph size of 
$\signalsize\!=\!\nrcluster \cdot \nodespercluster$ nodes. 
\begin{figure}[h]
	\centering
	\hspace*{0em}\includegraphics[width=.8\linewidth]{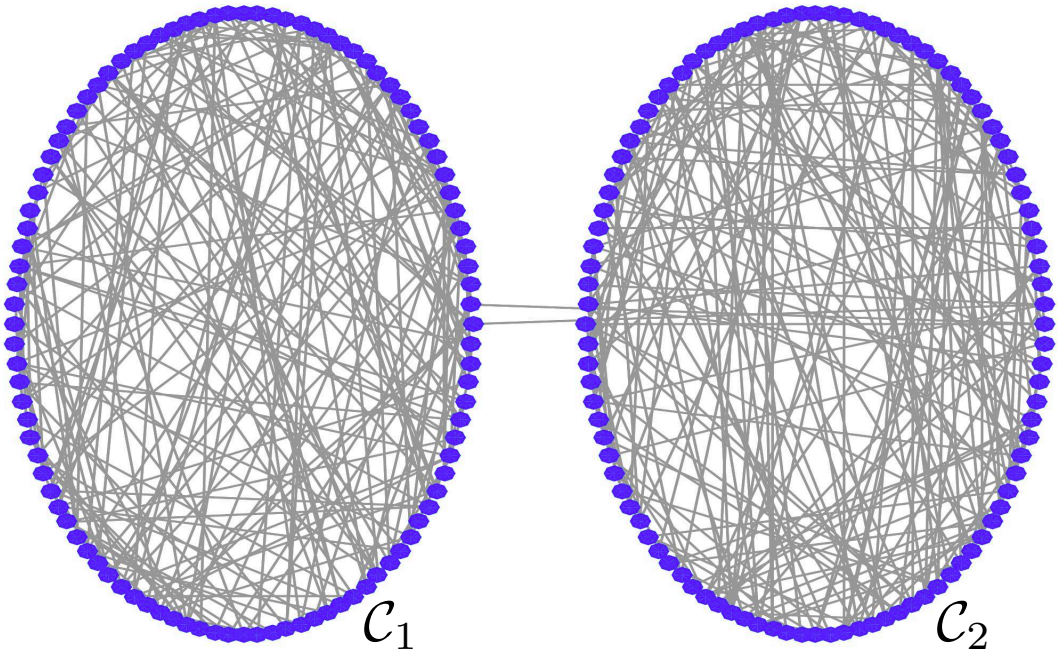}
	\captionsetup{justification=centering}
	\caption{Empirical graph $\mathcal{G}$ made up of $\nrcluster$ clusters 
	$\mathcal{C}_{1}$ and $\mathcal{C}_{2}$, each consisting of $\nodespercluster$ nodes.}.
	\label{fig_emp_grpah}
\end{figure}
The clusters are connected through few ``gate'' nodes. 
The maximum node degree of $\mathcal{G}$ is $d_{\rm max} \!=\!8$. 
Given the empirical graph $\mathcal{G}$, we generated a labeling $\vx^{(g)}$ 
by labeling the nodes for each cluster $\mathcal{C}_{c}$ by a 
random number $t_{c} \sim \mathcal{N}(0,1)$, i.e., $x^{(g)}_{i}\!=\!t_{c}$ 
for all nodes $i \in \mathcal{C}_{c}$ in the cluster $\mathcal{C}_{c}$. 
For each cluster $\mathcal{C}_{c}$ we assume that we are provided 
initial labels $y_{i}  = x_{i}$ for $\samplespercluster$ randomly 
choosen nodes $i \in \mathcal{C}_{c}$, giving rise to an overall 
sampling set $\mathcal{S}$ with $\samplesize\!=\! \nrcluster \cdot \samplespercluster$ 
nodes. We run Alg.\ \ref{nestrov_alg_acc} with initial smoothing parameter 
$\mu\!=\!1$, decreasing factor $\kappa\!=\!(2\cdot10^{-5})^{1/\numiter}$, 
error bound $\varepsilon \!\defeq\! \|\vx^{(g)} \|_{2} / 10^{5}$ 
and a fixed number of $\numiter$ iterations, to 
obtain the learned labels $\hat{x}_{i}$ for every 
node $i \in \mathcal{V}$. In Fig.\ \ref{fig_learnedlabeling}, we show 
the learned labeling $\hat{x}_{i}$ output by Alg.\ \ref{nestrov_alg_acc}. 
We also show the learned labeling $\hat{x}^{\rm LP}_{i}$ obtained 
using the well-known label progagation (LP) algorithm \cite[Alg.\ 11.1]{SemiSupervisedBook} 
which is run for the same number of iterations. 
\begin{figure}[h]
	\centering
	\hspace*{0em}\includegraphics[width=.8\linewidth]{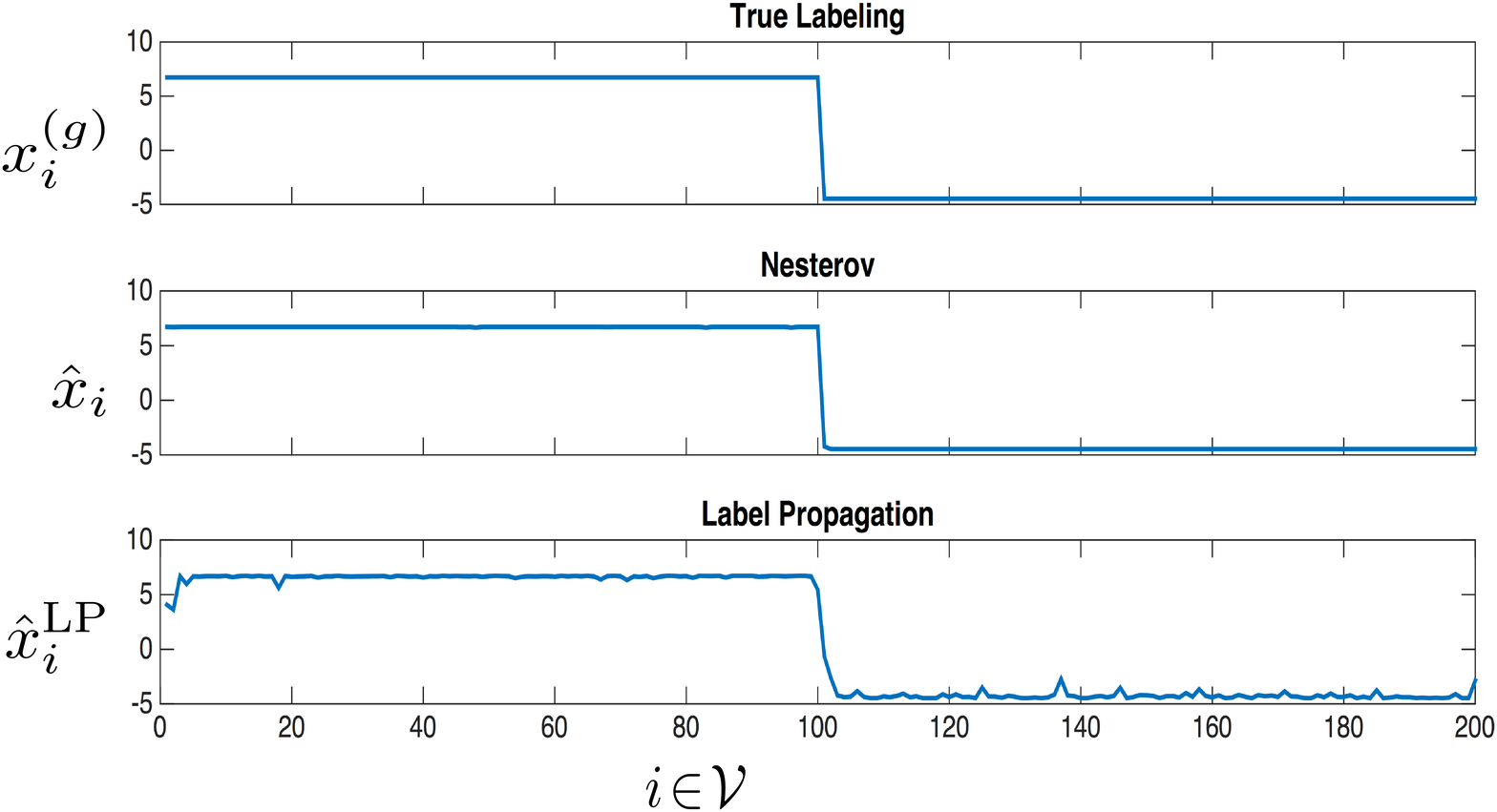}
	\captionsetup{justification=centering}
	\caption{True labels $x_{i}$ and labelings $\hat{x}_{i}$ and $\hat{x}^{\rm LP}_{i}$, 
	obtained by our method (Alg.\ \ref{nestrov_alg_acc}) and from LP.}.
	\label{fig_learnedlabeling}
\end{figure}
From Fig.\ \ref{fig_learnedlabeling}, it is evident that Alg.\ \ref{nestrov_alg_acc} 
yields better learning accuracy compared to plain LP, which is also reflected in 
the empirical normalized MSEs 
${\rm NMSE}_{\rm nest} \approx 2.1 \times 10^{-4}$ and ${\rm NMSE}_{\rm LP} \approx 2.4 \times 10^{-3}$  
obtained by averaging $\| \hat{\vx}\!-\!\vx^{(g)} \|^{2}_{2} / \| \vx^{(g)} \|^{2}_{2}$ and 
$\| \hat{\vx}^{\rm LP}\!-\!\vx^{(g)} \|^{2}_{2} / \| \vx^{(g)} \|^{2}_{2}$ over $100$ independent Monte Carlo runs. 
We have also depicted the dependence of the NMSE of Alg.\ \ref{nestrov_alg_acc} and 
LP on the iteration number $k$ in Fig.\ \ref{fig_convMSE}, which shows  
that after some inital phase, which comprises $\approx 100$ iterations, the NMSE obtained by 
Alg.\ \ref{nestrov_alg_acc} converges quickly to its stationary value. 
Remarkably, According to Fig.\ \ref{fig_convMSE}, the simple LP method provides 
smaller NMSE for the first few iterations. However, the comparison of the 
convergence speed of Alg.\ \ref{nestrov_alg_acc} and LP should be interpreted 
carefully, since the optimization problem underlying LP is based on the smooth 
Laplacian quadratic form \cite[Sec.\ 11.3.]{SemiSupervisedBook}, 
whereas Alg.\ \ref{nestrov_alg_acc} amounts to solving the 
non-smooth problem \eqref{equ_min_constr}. 
\begin{figure}[h]
	\centering
	\hspace*{0em}\includegraphics[width=.8\linewidth]{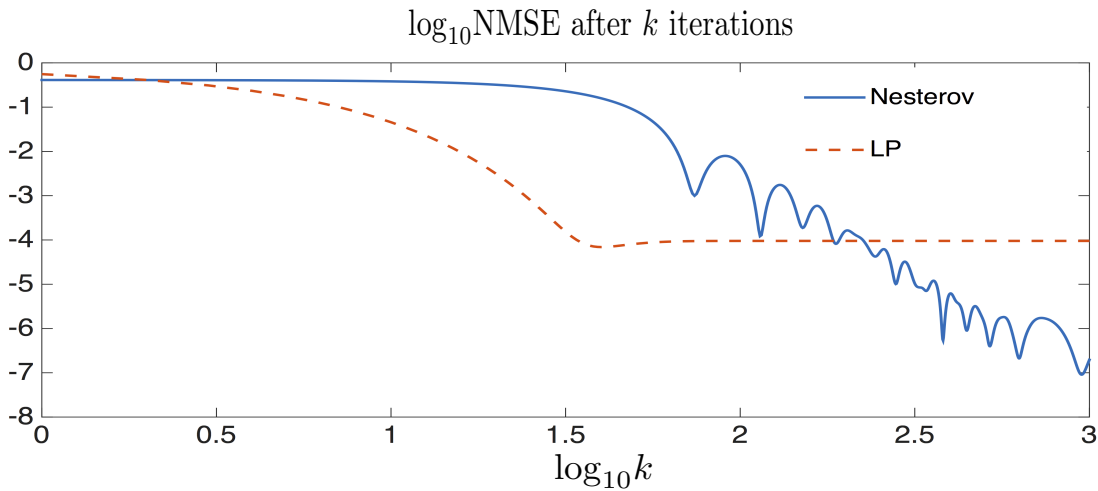}
	\captionsetup{justification=centering}
	\caption{Overall NMSE of Alg.\ \ref{nestrov_alg_acc} (${\rm NMSE}_{\rm nest}$) and label propagation (${\rm NMSE}_{\rm LP}$) vs.\ iteration number $k$.}.
	\label{fig_convMSE}
\end{figure}

\section{Conclusions}
The problem of semi-supervised learning from massive datastes over networks 
has been formulated as a nonsmooth convex optimization problem based on penalizing 
the total variation of the labeling. We applied the smoothing technique of 
Nesterov to this optimization 
problem for obtaining an efficient learning algorithm which can capitalize on 
huge amounts of unlabeled data using only few labeled datapoints. 
Moreover, we proposed an implementation of the learning 
method as message passing over the underlying data graph. This message passing algorithm 
can be easily implemented in a big data platform such as AKKA to 
allow for scalable learning algorithms. Future work includes the extension 
of the optimization framework to accomodate loss functions, different from the mean squared error, 
that better characterize the training error for discrete valued or categorial labels. 


\appendices


\section{Proof of Lemma \ref{lem_orthog_proj_closed_form}}
 \label{proof_lem_orthog_proj_closed_form}
We only detail the derivation of \eqref{equ_closed_form_step3}, since the derviation 
of \eqref{equ_closed_form_step4} is very similar. 
Our argument closely follows the derivations used in \cite[Sec. 3]{becker2011nesta}. 
Consider the constrained convex optimization problem \eqref{equ_orthog_proj_closed_form_constr_opt1}, 
which we repeat here for convenience:  
\begin{equation} 
\label{equ_consttr_appx_repeat}
\vv_k = \underset{\vx \in \mathcal{Q}}{\argmin} (\hat{L}/2) \| \vx- \vx_{k} \|_{2}^{2} \!+\! \vg^{T}_{k} (\vx\!-\!\vx_{k})
\end{equation} 
with constraint set $\mathcal{Q}\defeq\{ \vx : \emperr[\vx] \leq \varepsilon\} = \{ \vx : \big(\emperr[\vx]\big)^{2} \leq \varepsilon^2\}$. 
The Lagrangian associated with \eqref{equ_consttr_appx_repeat} is 
\begin{align}
\mathcal{L}(\vx,\lambda)&\!=\!(\hat{L}/2) \| \vx\!-\!\vx_{k} \|_{2}^{2} \!+\! \vg^{T}_{k} (\vx\!-\!\vx_{k}) \nonumber \\ 
& \!+\!\lambda( \big( \emperr[\vx]\big)^{2}\!-\!\varepsilon^{2}), 
\end{align}
and the corresponding KKT conditions for $\vv_{k}$ and $\lambda_{\varepsilon}$ 
to be primal and dual optimal read \cite[Section 5.5.3]{BoydConvexBook}
\begin{align}
\hat{L} (\vv_{k}\!-\!\vx_{k})\!+\!\vg_{k}\!+\!(\lambda_{\epsilon}/|\mathcal{S}|) \mathbf{D}(\mathcal{S}) (\vv_{k}\!-\!\vy)\!=\!\mathbf{0},\label{stationary}\\
\lambda_\varepsilon (\emperr[\vv_{k}]- \varepsilon) = 0,\label{equ_compslack}\\
\emperr[\vv_{k}]  \leq \varepsilon ,\label{error_bound}\\
\lambda_\varepsilon \geq 0,
\end{align}
with the diagonal matrix $\mathbf{D}(\mathcal{S})= \sum_{i \in \mathcal{S}} \mathbf{e}_{i} \mathbf{e}_{i}^{T}$. 
From condition \eqref{stationary}, we obtain 
\begin{equation}\label{stat2aaa}
\vv_{k}\!=\!\big(\mathbf{I}\!+\!\tilde{\lambda}_{\varepsilon} \mathbf{D}(\mathcal{S}) \big)^{-1} (\vx_{k}\!-\!(1/\hat{L}) \vg_{k}\!+\!\tilde{\lambda}_{\varepsilon} \mathbf{D}(\mathcal{S}) \vy).
\end{equation}
with 
\begin{equation}
 \tilde{\lambda}_{\varepsilon} \defeq \frac{\lambda_{\epsilon}}{\hat{L} |\mathcal{S}|} \label{equ_def_tilde_lambda}
\end{equation}
Using the elementary identity 
\begin{equation}
(\mathbf{I} + a \mathbf{D}(\mathcal{S}))^{-1} = \mathbf{I} - \frac{a}{1+a} \mathbf{D}(\mathcal{S})
\end{equation} 
which is valid for any $a \geq 0$, we can develop \eqref{stat2aaa} further to 
\begin{equation}\label{wsolved}
\vv_{k} = \Big(\mathbf{I} - \frac{ \tilde{\lambda}_{\varepsilon} }{1+\tilde{\lambda}_{\varepsilon} }\mathbf{D}(\mathcal{S}) \Big)(\vx_{k}-(1/\hat{L}) \vg_{k} + \tilde{\lambda}_{\varepsilon} \mathbf{D}(\mathcal{S}) \vy).
\end{equation}
Inserting \eqref{wsolved} into \eqref{error_bound} yields 
\begin{align}\label{error_inequ}
\emperr[\vv_{k}] & \stackrel{\eqref{wsolved}}{=} \emperr[ \big(\mathbf{I} - \frac{ \tilde{\lambda}_{\varepsilon} }{1+ \tilde{\lambda}_{\varepsilon} }\mathbf{D}(\mathcal{S}) \big)(\vx_{k}-(1/\hat{L}) \vg_{k} \nonumber \\ 
&+ \tilde{\lambda}_{\varepsilon} \mathbf{D}(\mathcal{S}) \vy)]  \nonumber \\ 
& \stackrel{\eqref{equ_def_emp_error}}{=} \frac{1}{(1+ \tilde{\lambda}_{\varepsilon})^2} \emperr[\vx_{k}-(1/\hat{L}) \vg_{k} - \vy)] \stackrel{\eqref{error_bound}}{\leq} \varepsilon.
\end{align}
From \eqref{error_inequ}, we have
\begin{equation}\label{eps_inequ}
\tilde{\lambda}_{\varepsilon} \geq  (1/\varepsilon)\emperr[\vx_{k}\!-\!(1/\hat{L}) \vg_{k}\!-\!\vy)]\!-\!1.
\end{equation}
Thus if  $(1/\varepsilon)\emperr[\vx_{k}\!-\!(1/\hat{L}) \vg_{k}\!-\!\vy)]\!>\!1$, then \eqref{eps_inequ} implies $\lambda_\varepsilon >0$, 
which, via \eqref{equ_compslack}, requires the inequality \eqref{error_inequ} to become an equality, i.e., 
\begin{equation}
\tilde{\lambda}_{\varepsilon}\!=\!(1/\varepsilon)\emperr[\vx_{k}\!-\!(1/\hat{L}) \vg_{k}\!-\!\vy)]\!-\!1
\end{equation}
This equality holds also if $ \emperr[\vx_{k}\!-\!(1/\hat{L}) \vg_{k}\!-\!\vy)]/\varepsilon =1$.
For $ \emperr[\vx_{k}\!-\!(1/\hat{L}) \vg_{k}\!-\!\vy)]/\varepsilon  <1$, complementary slackness 
\eqref{equ_compslack} requires $\tilde{\lambda}_{\varepsilon} = 0$. 
Thus the optimal dual variable $\tilde{\lambda}_{\varepsilon}$ is fully determined by the quantity 
$\emperr[\vx_{k}-(1/\hat{L}) \vg_{k} - \vy)]$ via 
\begin{equation}
\label{eq_C}
\tilde{\lambda}_\varepsilon = \max\{0,(1/\varepsilon)\emperr[\vx_{k}-(1/\hat{L}) \vg_{k} - \vy)]\!-\!1\}.
\end{equation}

\newpage
\bibliographystyle{abbrv}
\bibliography{LitAJ_JournalCvx,tf-zentral}

\end{document}